\documentclass{article} 

\usepackage{nips11submit_e,times}

\usepackage{amsmath,amssymb}
\usepackage{dsfont}
\usepackage{verbatim}
\usepackage{multido}
\usepackage{rotating}
\usepackage{graphicx}
\usepackage{hyperref}

\newcommand{\Ecal}{\mathcal{E}}
\newcommand{\Mcal}{\mathcal{M}}
\newcommand{\Pcal}{\mathcal{P}}
\newcommand{\Xcal}{\mathcal{X}}
\newcommand{\Ycal}{\mathcal{Y}}

\newcommand{\ol}{\overline}

\graphicspath{{../Figures/}}

\newtheorem{theorem}{Theorem}[section]
\newtheorem{lemma}[theorem]{Lemma}
\newtheorem{corollary}[theorem]{Corollary}
\newenvironment{proof}[1][Proof]{\begin{trivlist}
\item[\hskip \labelsep {\bfseries #1}]}{\end{trivlist}}

\title{Expressive Power and Approximation Errors of Restricted Boltzmann Machines}

\author{
Guido F.~Mont\'ufar$^1$, Johannes Rauh$^1$, and Nihat Ay$^{1,2}$ \\
$^1$Max Planck Institute for Mathematics in the Sciences, Inselstra\ss e 22 04103 Leipzig, Germany\\
$^2$Santa Fe Institute, 1399 Hyde Park Road, Santa Fe, New Mexico 87501, USA\\
\texttt{$\{$montufar,jrauh,nay$\}$@mis.mpg.de} 
}

\newcommand{\RBM}{\ensuremath{\operatorname{RBM}}}

\nipsfinalcopy 

\begin{document}

\maketitle

\begin{abstract}
  We present explicit classes of probability distributions that can be learned by Restricted Boltzmann Machines (RBMs) depending on the number of units that they contain, and which are representative for the expressive power of the model. We use this to show that the maximal Kullback-Leibler divergence to the RBM model with $n$ visible and $m$ hidden units is bounded from above by 
  $n - \left\lfloor \log(m+1) \right\rfloor - \frac{m+1}{2^{\left\lfloor \log(m+1) \right\rfloor}}
  \approx (n -1) - \log (m+1)$. In this way we can specify the number of hidden units that guarantees a sufficiently rich model containing different classes of distributions and respecting a given error tolerance. 
\end{abstract}

\section{Introduction}
\label{sec:introduction}
A Restricted Boltzmann Machine (RBM)~\cite{Smolensky1986,Freund1992} is a learning system consisting of two layers of binary stochastic units, a hidden layer and a visible layer, with a complete bipartite interaction
graph. RBMs are used as generative models to simulate input distributions of binary data. They can be trained in an unsupervised way and more efficiently than general Boltzmann Machines, which are not restricted to have a bipartite interaction graph~\cite{Hinton2002, Carreira2005}. Furthermore, they can be used as building blocks to progressively train and study deep learning systems~\cite{Hinton2006,Bengio2007,LeRoux2010,Montufar2011}. Hence, RBMs have received increasing attention in the past years.

An RBM with $n$ visible and $m$ hidden units generates a stationary distribution on the states of the visible units which has the following form: 
\begin{equation*}
 p_{{}_{W,C,B}}(v) = \frac{1}{Z_{{}_{W,C,B}}}\sum_{h\in\{0,1\}^m} \exp\left( h^\top W v + C^\top h + B^\top v \right)\quad \forall v\in\{0,1\}^n \;,
\end{equation*}
where $h\in\{0,1\}^{m}$ denotes a state vector of the hidden units, $W\in\mathbf{R}^{m\times n}, C\in\mathbf{R}^{m}$ and
$B\in\mathbf{R}^{n}$ constitute the model parameters, and $Z_{{}_{W,C,B}}$ is a corresponding normalization constant.
In the sequel we denote by $\RBM_{n,m}$ the set of all probability distributions on $\{0,1\}^n$ which can be
approximated arbitrarily well by a visible distribution generated by the RBM with $m$ hidden and $n$ visible units for
an appropriate choice of the parameter values.

As shown in~\cite{Montufar2011} (generalizing results from~\cite{LeRoux2008}) $\RBM_{n,m}$ contains any probability distribution if $m\geq 2^{n-1}+1$. 
On the other hand, if $\RBM_{n,m}$ equals the set $\Pcal$ of all probability distributions on
$\{0,1\}^n$, then it must have at least $\dim(\Pcal)=2^n-1$ parameters, 
and thus at least $\left\lceil2^n/(n+1)\right\rceil -1$ hidden units~\cite{Montufar2011}. 
In fact, in~\cite{Cueto2010} it was shown that for most combinations of $m$ and $n$ the dimension of $\RBM_{n,m}$ (as a manifold, possibly with singularities) equals either the number of parameters or $2^n-1$, whatever is smaller. 
However, the geometry of $\RBM_{n,m}$ is intricate, and even an RBM of
dimension $2^n-1$ is not guaranteed to contain all visible distributions, see~\cite{Montufar2010} for counterexamples. 

In summary, an RBM that can approximate any distribution arbitrarily well must have a very large number of parameters and hidden
units. In practice, training such a large system is not desirable or even possible. 
However, there are at least two reasons why in many cases this is not necessary:
\begin{itemize}
 \item An appropriate approximation of distributions is sufficient for most purposes.
 \item The interesting distributions the system shall simulate belong to a small class of distributions. Therefore, the model does not need to approximate all distributions.
\end{itemize}

For example, the set of optimal policies in reinforcement learning~\cite{Sutton1998}, the set of dynamics kernels that maximize
predictive information in robotics~\cite{Zahedi2010} or the information flow in neural networks~\cite{Ay2003} are contained in very low
dimensional manifolds; see~\cite{AMR2011}. 
On the other hand, usually it is very hard to mathematically describe a set containing the optimal solutions to general problems, or a set
of interesting probability distributions (for example the class of distributions generating natural images).
Furthermore, although RBMs are parametric models and for any choice of the parameters we have a resulting probability
distribution, in general it is difficult to explicitly specify this resulting probability distribution (or even to
estimate it~\cite{Long2010}). 
Due to these difficulties the number of hidden units $m$ is often chosen on the basis of experience~\cite{Hinton2010}, or $m$ is considered as a hyperparameter which is optimized by extensive search, depending on the distributions to be simulated by the RBM. 

In this paper we give an explicit description of classes of distributions that are contained in $\RBM_{n,m}$, and which are representative for the expressive power of this model. 
Using this description, we estimate the maximal Kullback-Leibler divergence between an arbitrary probability distribution and the best approximation within $\RBM_{n,m}$. 

This paper is organized as follows: Section~\ref{sec:approximation-error} discusses the different kinds of errors that appear when an RBM learns. Section~\ref{sec:model-classes} introduces the statistical models studied in this paper. 
Section~\ref{sec:what-RBMS-can-learn} studies submodels of $\RBM_{n,m}$.  An upper bound of the approximation error for RBMs is found in Section~\ref{sec:maxim-appr-errors}.

\section{Approximation Error}
\label{sec:approximation-error}
When training an RBM to represent a distribution $p$, there are mainly three contributions to the discrepancy between $p$ and the state of the RBM after training: 
\begin{enumerate}
\item Usually the underlying distribution $p$ is unknown and only a set of samples generated by $p$ is observed. 
These samples can be represented as an empirical distribution $p^{\text{Data}}$, which usually is not identical with $p$. 
\item The set $\RBM_{n,m}$ does not contain every probability distribution, unless the number of hidden units is very large, as we outlined in the introduction. 
Therefore, we have an approximation error given by the distance of $p^{\text{Data}}$ to the best approximation $p^{\text{Data}}_{\operatorname{RBM}}$ contained in the RBM model. 
\item The learning process may yield a solution $\tilde p^{\text{Data}}_{\operatorname{RBM}}$ in $\operatorname{RBM}$ which is not the optimum $p^{\text{Data}}_{\operatorname{RBM}}$. This occurs, for example, if the learning algorithm gets trapped in a local optimum, or if it optimizes an objective different from Maximum Likelihood, e.g. contrastive divergence (CD), see~\cite{Carreira2005}. 
\end{enumerate}

In this paper we study the expressive power of the RBM model and the Kullback-Leibler divergence from an arbitrary distribution to its best representation within the RBM model. 
Estimating the approximation error is difficult, because the geometry of the $\operatorname{RBM}$ model is not sufficiently understood. 
Our strategy is to find subsets $\Mcal\subseteq\RBM_{n,m}$ that are easy to describe. 
Then the maximal error when approximating probability distributions with an $\RBM$ is upper bounded by the maximal error when approximating with
$\Mcal$. 

Consider a finite set $\mathcal{X}$. A real valued function on $\mathcal{X}$ can be seen as a real vector with
$|\mathcal{X}|$ entries. The set $\Pcal=\Pcal(\Xcal)$ of all probability distributions on $\mathcal{X}$ is a
$(|\mathcal{X}|-1)$-dimensional simplex in $\mathbf{R}^{|\Xcal|}$.
There are several notions of distance between probability distributions, and in turn for the error in the representation
(approximation) of a probability distribution. One possibility is to use the induced distance of the Euclidean space
$\mathbf{R}^{|\Xcal|}$.
From the point of view of information theory, a more meaningful distance notion for probability distributions is the
Kullback-Leibler divergence:
\begin{equation*}
 D(p\|q):=\sum_x p(x) \log\frac{p(x)}{q(x)}\;.
\end{equation*}
In this paper we use the basis $2$ logarithm. The Kullback-Leibler (KL) divergence is non-negative and vanishes if and only
if $p=q$. If the support of $q$ does not contain the support of $p$ it is defined as $\infty$. The summands with $p(x)=0$ are set to $0$. 
The KL-divergence is not symmetric, but it has nice information theoretic properties~\cite{KL1951,CT2006}. 

If $\Ecal\subseteq\Pcal$ is a statistical model and if $p\in\Pcal$, then any probability distribution $p_{\Ecal}\in\ol\Ecal$ satisfying 
\begin{equation*}
  D(p\|p_{\Ecal}) = D(p\|\Ecal) := \min\{ D(p\|q) : q\in\ol\Ecal \}
\end{equation*}
is called a \emph{(generalized) reversed information projection}, or $rI$-projection. 
Here, $\ol\Ecal$ denotes the closure of $\Ecal$. If $p$ is an empirical distribution, then one can show that any $rI$-projection is a maximum
likelihood estimate.

In order to assess an RBM or some other model $\mathcal{M}$ we use the maximal approximation error with respect to the KL-divergence when approximating arbitrary probability distributions using $\mathcal{M}$:
\begin{equation*}
  D_\mathcal{M}:=\max\left\{ D(p\|\mathcal{M}):p\in\Pcal \right\}\;.
\end{equation*}
For example, the maximal KL-divergence to the uniform distribution $\tfrac{\mathds{1}}{|\mathcal{X}|}$ is attained by the Dirac delta
distributions $\delta_x$, $x\in\mathcal{X}$, and amounts to:
\begin{equation}
  \label{eq:KL-from-unidist}
  D_{\left\{\frac{\mathds{1}}{|\mathcal{X}|}\right\}}= D(\delta_x\|\tfrac{\mathds{1}}{|\mathcal{X}|})=\log|\Xcal|\;.
\end{equation}

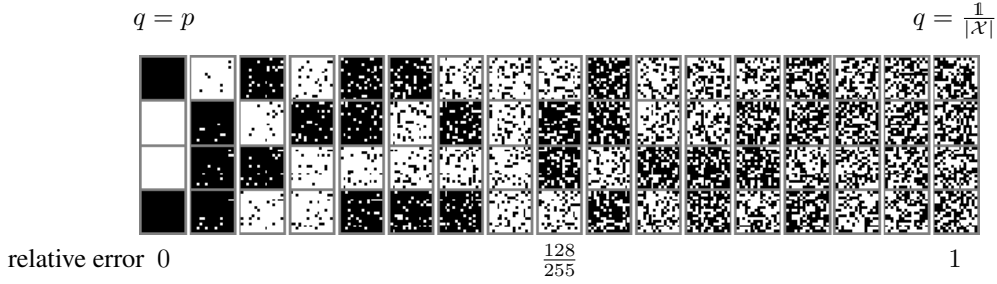
\begin{figure}
\setlength{\unitlength}{\textwidth}
\begin{picture}(1,.23)(.07,0)
\put(0.2,0.03){\includegraphics[width=.8\textwidth]{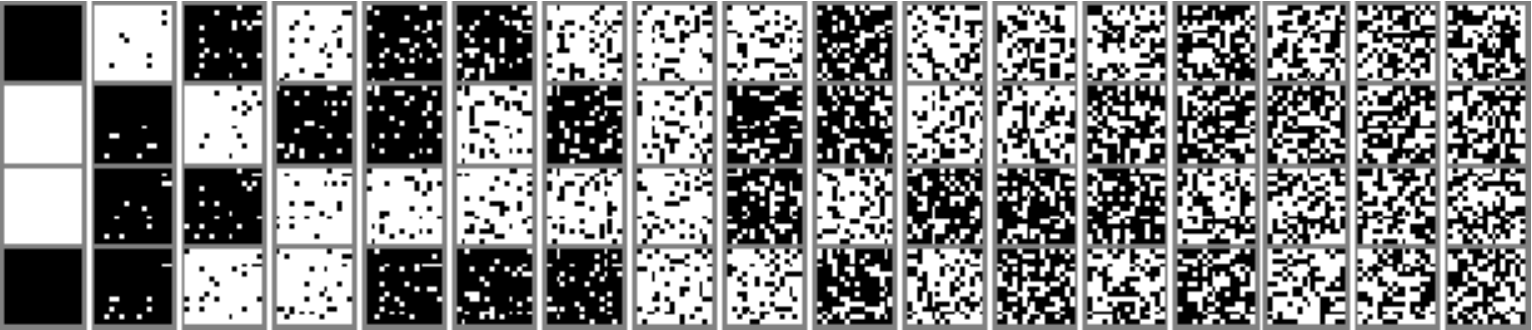}}

\newlength{\ta}
\settowidth{\ta}{relative error}
\put(.14,0){\hspace{-.5\ta}relative error}
\newlength{\tss}
\settowidth{\tss}{$0$}
\newlength{\tsss}
\settowidth{\tsss}{$\tfrac{128}{255}$}
\newlength{\tssss}
\settowidth{\tssss}{$1$}

\newlength{\tsssss}
\settowidth{\tsssss}{$q=p$}
\put(.2235,0.23){\hspace{-.5\tsssss}$q=p$}
\newlength{\tssssss}
\settowidth{\tssssss}{$q=\tfrac{\mathds{1}}{|\mathcal{X}|}$}
\put(.9765,0.23){\hspace{-.5\tssssss}$q=\tfrac{\mathds{1}}{|\mathcal{X}|}$}
\put(.2235,0){\hspace{-.5\tss}$0$} 
\put(.6,0){\hspace{-.5\tsss}$\tfrac{128}{255}$}
\put(.9765,0){\hspace{-.5\tssss}$1$}
\end{picture} 
\vspace{0.1cm}
\caption{%
This figure gives an intuition on what the size of an error means for probability distributions on images with $16\times 16$ pixels. 
Every column shows four samples drawn from the best approximation $q$ of the distribution $p=\tfrac{1}{2}(\delta_{(1\dots1)} + \delta_{(0\dots0)})$ within a partition model with 2 randomly chosen cubical blocks, containing $(0\dots0)$ and $(1\dots1)$, of cardinality from $1$ (first column) to $\tfrac{|\mathcal{X}|}{2}$ (last column). 
As a measure of error ranging from $0$ to $1$ we take ${D(p\|q)}/{D\big(p\|\tfrac{\mathds{1}}{|\mathcal{X}|}\big)}$. 
The last column shows samples from the uniform distribution, which is, in particular, the best approximation of $q$ within $\RBM_{n,0}$. 
Note that an RBM with $1$ hidden unit can approximate $q$ with arbitrary accuracy, see~Theorem~\ref{theorem1}.}\label{fig2}
\end{figure}

\section{Model Classes}
\label{sec:model-classes}

\subsection{Exponential families and product measures}
In this work we only need a restricted class of exponential families, namely exponential families on a finite set with
uniform reference measure.  See~\cite{Brown86:Fundamentals_of_Exponential_Families} for more on exponential families. The boundary of discrete exponential families is discussed in~\cite{RKA10:Support_Sets_and_Or_Mat}, which uses a similar notation. 

Let $A\in\mathbf{R}^{d\times |\mathcal{X}|}$ be a matrix.  The columns $A_{x}$ of $A$ will be indexed by $x\in\Xcal$.
The rows of $A$ can be interpreted as functions on $\mathbf{R}$.
The \emph{exponential family} $\mathcal{E}_{A}$ with \emph{sufficient statistics $A$} consists of all probability
distributions of the form $p_{\lambda}$, $\lambda\in\mathbf{R}^{d}$, where
\begin{equation*}
  p_{\lambda}(x) = \frac{\exp(\lambda^{\top} A_x)}{\sum_x \exp(\lambda^{\top} A_x)},\qquad\text{for all }x\in\Xcal.
\end{equation*}
Note that any probability distribution in $\mathcal{E}_{A}$ has full support. Furthermore, $\mathcal{E}_{A}$ is in general
not a closed set. The closure $\ol{\mathcal{E}_{A}}$ (with respect to the usual topology on $\mathbf{R}^{\Xcal}$) will be important in the following. 
Exponential families behave nicely with respect to $rI$-projection: Any $p\in\Pcal$ has a unique $rI$-projection $p_{\Ecal}$ to $\ol{\Ecal_A}$. 

The most important exponential families in this work are the independence models. The \emph{independence model} of $n$ binary random variables consists of all probability distributions on $\{0,1\}^{n}$ that factorize:
\begin{equation*}
  \ol{\Ecal_{n}} = \Big\{p\in\mathcal{P}(\Xcal): p(x_1,\ldots,x_n)=\prod_{i=1}^n p_i(x_i)\text{ for some }p_{i}\in\Pcal(\{0,1\}) \Big\}\;.
\end{equation*}
It is the closure of an $n$-dimensional exponential family $\Ecal_{n}$. This model corresponds to the RBM model with no hidden units. 
An element of the independence model is called a \emph{product distribution}. 

\begin{lemma}[Corollary 4.1 of~\cite{AyKnauf06:Maximizing_Multiinformation}]
  Let $\ol{\Ecal_{n}}$ be the independence model on $\{0,1\}^{n}$.  If $n>0$, then $D_{\Ecal_{n}}=(n-1)$. 
  The global maximizers are the distribution of the form $\frac12(\delta_{x} + \delta_{y})$, where
  $x,y\in\{0,1\}^{n}$ satisfy $x_{i}+y_{i}=1$ for all $i$.
\end{lemma}

This result should be compared with~\eqref{eq:KL-from-unidist}.  Although the independence model is much larger than the set $\{\tfrac{\mathds{1}}{|\mathcal{X}|}\}$, the maximal divergence decreases only by $1$. 
As shown in~\cite{Rauh11:Thesis}, if $\Ecal$ is any exponential family of dimension $k$, then $D_{\Ecal}\ge \log(|\Xcal|/(k+1))$. Thus, this notion of distance is rather strong. 
The exponential families satisfying $D_\Ecal = \log(|\Xcal|/(k+1))$ are partition models; they will be defined in the following section. 

\subsection{Partition models and mixtures of products with disjoint supports}
The \emph{mixture} of $m$ models $\Mcal_{1},\dots,\Mcal_{m}\subseteq\Pcal$ is the set of all convex combinations
\begin{equation}
  \label{eq:mixture_element}
  p=\sum_i \alpha_i p_i\;, \text{ where } p_{i}\in\Mcal_{i}, \alpha_i\geq0, \sum_i\alpha_i =1\;.
\end{equation}
In general, mixture models are complicated objects. 
Even if all models $\Mcal_{1}=\dots=\Mcal_{m}$ are equal, it is
difficult to describe the mixture~\cite{Lindsay1995,Montufar2010a}. 
The situation simplifies considerably if the models have disjoint
supports. 
Note that given any partition $\xi=\{\Xcal_{1},\dots,\Xcal_{m}\}$ of $\Xcal$, any $p\in\Pcal$ can be written as $p(x) = p^{\Xcal_{i}}(x)p(\Xcal_{i})$ for all $x\in\Xcal_{i}$ and $i\in\{1,\dots,m\}$, where $p^{\Xcal_{i}}$ is a probability measure in $\Pcal(\Xcal_{i})$ for all $i$. 
\begin{lemma}
  \label{lem:rI-of-mixture}
  Let $\xi=\{\Xcal_{1},\dots,\Xcal_{m}\}$ be a partition of $\Xcal$ and let $\Mcal_{1},\dots,\Mcal_{m}$ be statistical models such that $\Mcal_{i}\subseteq\Pcal(\Xcal_{i})$. 
Consider any $p\in\Pcal$ and corresponding $p^{\Xcal_{i}}$ such that $p(x) = p^{\Xcal_{i}}(x)p(\Xcal_{i})$ for $x\in\Xcal_i$. 
Let $p_{i}$ be an $rI$-projection of $p^{\Xcal_{i}}$ to $\Mcal_{i}$. 
 Then the $rI$-projection $p_{\Mcal}$ of $p$ to the mixture $\Mcal$ of $\Mcal_{1},\dots,\Mcal_{m}$ satisfies
  \begin{equation*}
    p_{\Mcal}(x) = p(\Xcal_{i})p_{i}(x),\qquad\text{ whenever }x\in\Xcal_{i}\;.
  \end{equation*}
  Therefore, $D(p\|\Mcal) = \sum_i p(\Xcal_i) D(p^{\Xcal_i}\|\Mcal_i)$, and so $D_{\Mcal} = \max_{i=1,\dots,m} D_{\Mcal_{i}}$. 
\end{lemma}
\begin{proof}
  Let $p\in\Mcal$ be as in~\eqref{eq:mixture_element}. 
  Then $D(q\|p) = \sum_{i=1}^{m}q(\Xcal_{i}) D(q^{\Xcal_{i}}\|p_{i})$
  for all $q\in\Pcal$. For fixed $q$ this sum is minimal if and only if each term is minimal. 
\hfill$\square$\end{proof}
If each $\Mcal_{i}$ is an exponential family, then the mixture is also an exponential family (this is not true if the supports of the models $\Mcal_{i}$ are not disjoint). 
In the rest of this section we discuss two examples. 

If each $\Mcal_{i}$ equals the set containing just the uniform distribution on $\Xcal_{i}$, then $\Mcal$ is called the
\emph{partition model} of $\xi$, denoted with $\Pcal_{\xi}$. 
The partition model $\Pcal_\xi$ is given by all distributions with constant value on each block $\mathcal{X}_i$, i.e. those that satisfy
 $p(x)=p(y)$ for all $x,y\in\mathcal{X}_i$. 
This is the closure of the exponential family with sufficient statistics
\begin{equation*}
 A_x=\left(
    \chi_{1}(x),
    \chi_{2}(x),
\dots,
    \chi_{d}(x)\right)^{\top}\;,
\end{equation*}
where $\chi_i:=\chi_{{}_{\mathcal{X}_i}}$ is $1$ on $x\in\mathcal{X}_i$, and $0$ everywhere else. See~\cite{Rauh11:Thesis} for interesting properties of partition models. 

The partition models include the set of finite exchangeable distributions (see e.g.~\cite{Diaconis1980}), where the blocks of the partition are the sets of binary vectors which have the same number of entries equal to one. The probability of a vector $v$ depends only on the number of ones, but not on their position. 

\begin{figure}[!ht]
\begin{center}
\vspace{-.5cm}
\includegraphics[trim=6cm 21.6cm 5.5cm 4.1cm, clip=true, width=11cm]{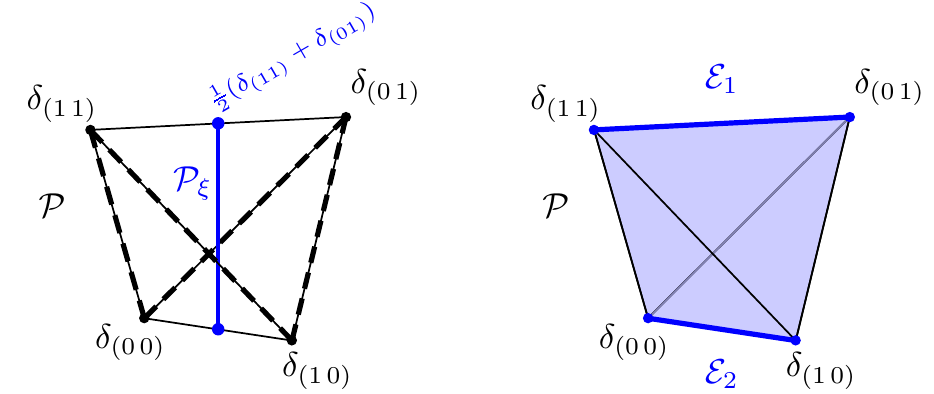}  
\end{center}
\vspace{-.1cm}
\caption{Models in $\Pcal(\{0,1\}^2)$. 
Left: The blue line represents the partition model $\mathcal{P}_\xi$ with partition $\xi=\{(1 1), (0 1)\}\cup\{(0 0),(1 0)\}$. 
The dashed lines represent the set of KL-divergence maximizers for $\mathcal{P}_\xi$. 
Right: The mixture of the product distributions $\mathcal{E}_1$ and $\mathcal{E}_2$ with disjoint supports on $\{(11),(01)\}$ and $\{(00),(10)\}$ corresponding to the same partition $\xi$ equals the whole simplex $\Pcal$.}\label{fig1}
\end{figure}

\begin{corollary}
  \label{cor:part-mod-max-KL}
  Let $\xi=\{\Xcal_{1},\dots,\Xcal_{m}\}$ be a partition of $\Xcal$.  Then
  $D_{\Pcal_{\xi}} = \max_{i=1,\dots,m}\log|\Xcal_{i}|$.
\end{corollary}

Now assume that $\mathcal{X}=\{0,1\}^{n}$ is the set of binary vectors of length $n$.
As a subset of $\mathbf{R}^n$ it consists of the vertices (extreme points) of the $n$-dimensional hypercube. 
The vertices of a $k$-dimensional face of the $n$-cube are given by fixing the values of $x$ in $n-k$ positions: 
\begin{equation*}
 \{x\in\{0,1\}^n:x_{i}=\tilde x_i, \forall i\in I, \mbox{ for some } I\subseteq\{1,\ldots,n\}, |I|=n-k\}
\end{equation*}
We call such a subset $\Ycal\subseteq\Xcal$ \emph{cubical} or a \emph{face} of the $n$-cube. A cubical subset of cardinality $2^{k}$ can be naturally identified with $\{0,1\}^{k}$. This identification allows to define independence models and product measures on $\Pcal(\Ycal)\subseteq\Pcal(\Xcal)$. 
Note that product measures on $\Ycal$ are also product measures on $\Xcal$, and the independence model on $\Ycal$ is a subset of the independence model on $\Xcal$.

\begin{corollary}
  \label{cor:ind-mix-max-KL}
  Let $\xi=\{\Xcal_{1},\dots,\Xcal_{m}\}$ be a partition of $\Xcal=\{0,1\}^{n}$ into cubical sets.  For any $i$ let
  $\Ecal_{i}$ be the independence model on $\Xcal_{i}$, and let $\Mcal$ be the mixture of $\Ecal_{1},\dots,\Ecal_{m}$.
  Then
  \begin{equation*}
    D_{\Mcal} = \max_{i=1,\dots,m}\log(|\Xcal_{i}|)-1\;.
  \end{equation*}
\end{corollary}
See Figure~\ref{fig2} for an intuition on the approximation error of partition models, and see Figure~\ref{fig1} for small examples of a partition model and of a mixture of products with disjoint support. 

\section{Classes of distributions that RBMs can learn}
\label{sec:what-RBMS-can-learn}

Consider a set $\xi=\{\mathcal{X}_i\}_{i=1}^{m}$ of $m$ disjoint cubical sets $\mathcal{X}_i$ in $\mathcal{X}$. Such a $\xi$ is a partition of some subset $\cup\xi=\cup_{i}\Xcal_{i}$ of $\mathcal{X}$ into $m$ disjoint cubical sets. 
We write $G_{m}$ for the collection of all such partitions. 
We have the following result:

\begin{theorem}\label{theorem1}
$\RBM_{n,m}$ contains the following distributions:
\begin{itemize}
\item Any mixture of one arbitrary product distribution, $m-k$ product distributions with support on arbitrary but disjoint faces of the $n$-cube, and $k$ arbitrary distributions with support on any edges of the $n$-cube, for any $0\leq k\leq m$. 
In particular:
\item Any mixture of $m+1$ product distributions with disjoint cubical supports. 
In consequence, $\RBM_{n,m}$ contains the partition model of any partition in $G_{m+1}$. 
\end{itemize}
\end{theorem}

Restricting the cubical sets of the second item to edges, i.e. pairs of vectors differing in one entry, we see that the above theorem implies the following previously known result, which was shown in~\cite{Montufar2011}:

\vspace{-.1cm}
\begin{corollary}\label{corollary1}
$\RBM_{n,m}$ contains the following distributions:
 \begin{itemize}
 \item 
Any distribution with a support set that can be covered by $m+1$ pairs of vectors differing in one entry. In particular, this includes:
 \item
Any distribution in $\mathcal{P}$ with a support of cardinality smaller than or equal to $m+1$. 
\end{itemize}
\end{corollary}

Corollary~\ref{corollary1} implies that an RBM with $m\geq 2^{n-1}-1$ hidden units is a universal approximator of
distributions on $\{0,1\}^n$, i.e. can approximate any distribution to an arbitrarily good accuracy. 

Assume $m+1=2^{k}$ and let $\xi$ be a partition of $\Xcal$ into $m+1$ disjoint cubical sets of equal size. 
Let us denote by $\mathcal{P}_{\xi,1}$ the set of all distributions which can be written as a mixture of $m+1$ product distributions with
support on the elements of $\xi$. 
The dimension of $\mathcal{P}_{\xi,1}$ is given by
\begin{equation*}
  \dim \mathcal{P}_{\xi,1} = (m+1) \log\left(\frac{2^n}{m+1}\right) + m+1 + n= (m+1)\cdot n + (m+1) + n - (m+1)\log(m+1)\;.
\end{equation*}
The dimension of the set of visible distribution represented by an RBM is at most equal to the number of parameters, see~\cite{Montufar2011}, this is $m\cdot n + m + n$. 
This means that the class given above has roughly the same dimension as the set of distributions that can be represented. 
In fact,
\begin{equation*}
  \dim \mathcal{P}_{\xi,1}-\dim \operatorname{RBM}_{m-1} = 
  n+1-(m+1)\log(m+1)\;.
\end{equation*}
This means that the class of distributions $\mathcal{P}_{\xi,1}$ which by Theorem~\ref{theorem1} can be represented by
$\RBM_{n,m}$ is not contained in $\RBM_{n,m-1}$ when $(m+1)^{m+1}\leq 2^{n+1}$. 

\begin{proof}[Proof of Theorem~\ref{theorem1}]
The proof draws on ideas from~\cite{LeRoux2008} and~\cite{Montufar2011}. 
An RBM with no hidden units can represent precisely the independence model, i.e. all product distributions, and in particular any uniform distribution on a face of the
$n$-cube. 

Consider an RBM with $m-1$ hidden units. For any choice of the parameters $W\in\mathbf{R}^{m-1\times n}, B\in\mathbf{R}^{n}, C\in\mathbf{R}^{m-1}$ we can write the resulting distribution on the visible units as: 
\begin{equation}
 p(v)=\frac{\sum_h z(v,h) }{\sum_{v',h'} z(v',h')}\;, \label{equationuno}
\end{equation}
where $z(v,h)=\exp(h W v + B v + C h)$. 
Appending one additional hidden unit, 
with connection weights $w$ to the visible units and bias $c$, produces a new distribution which can be written as
follows: 
\begin{equation*}
 p_{w,c}(v)=\frac{(1+\exp(w v + c))\sum_h z(v,h) }{\sum_{v',h'} (1+\exp(w v' + c )) z(v',h')}\;.\label{equationunoo}
\end{equation*}
Consider now any set $I\subseteq[n]:=\{1,\ldots,n\}$ and an arbitrary visible vector $u\in\mathcal{X}$. The values of $u$ in the positions $[n]\setminus I$ define a face $F:=\{v\in\mathcal{X}: v_i=u_i\;,\forall i \not\in I\}$ of the $n$-cube of dimension~$|I|$. 
 Let $\mathds{1}:=(1,\ldots,1)\in\mathbf{R}^n$ and denote by $u^{I,0}$ the vector with entries $u^{I,0}_i=u_i, \forall i\not\in I$ and $u^{I,0}_i=0, \forall i\in I$.
 Let $\lambda^I\in\mathbf{R}^n$ with $\lambda^I_i=0\;,\forall i \not\in I$ and let $\lambda_{c},a\in\mathbf{R}$.  Define
 the connection weights $w$ and $c$ as follows:
\begin{gather*}
 w=a(u^{I,0}-\frac{1}{2}\mathds{1}^{I,0}) + \lambda^I\;,\\
c=- a(u^{I,0}-\frac{1}{2}\mathds{1}^{I,0})^\top u + \lambda_c\;.
\end{gather*}
For this choice and $a\to\infty$ equation~(\ref{equationunoo}) yields: 
\begin{equation}
 p_{w,c}(v)=\begin{cases}
                                \frac{p(v)}{1 + \sum_{v'\in F} \exp{(\lambda^I\cdot v' + \lambda_c )} p (v')}, & \forall v\not\in F\\
                                \frac{(1+ \exp({\lambda^I\cdot v + \lambda_c})) p(v)}{1 + \sum_{v'\in F} \exp{(\lambda^I\cdot v' + \lambda_c )} p (v')}, & \forall v \in F
                               \end{cases} \label{unaeq}\;.
\end{equation}

If the initial $p$ from equation~(\ref{equationuno}) is such that its restriction to $F$ is a product distribution,
then 
$p(v)=K \exp(\eta^I\cdot v)\;,\forall v\in F$, where $K$ is a constant and $\eta^I$ is a vector with
$\eta^I_i=0\;,\forall i\not\in I$.  We can choose $\lambda^I=\beta^I-\eta^I$, and
$\exp(\lambda_c)=\alpha\frac{1}{K\sum_{v\in F} \exp(\beta^I \cdot v)}$. For this choice, equation~(\ref{unaeq}) yields:
\begin{equation*}
 p_{w,c} = (\alpha-1) p + \alpha \hat p\;, 
\end{equation*}
where $\hat p$ is a product distribution with support in $F$ and arbitrary natural parameters $\beta^I$, and $\alpha$ is an arbitrary mixture weight in $[0,1]$. 
Finally, the product distributions on edges of the cube are arbitrary, see~\cite{Montufar2010a} or \cite{Montufar2011} for details, and hence the restriction of any $p$ to any edge is a product distribution. 
\hfill$\square$\end{proof}

\section{Maximal Approximation Errors of RBMs}
\label{sec:maxim-appr-errors}
Let $m < 2^{n-1}-1$. By Theorem~\ref{theorem1} all partition models for partitions of $\{0,1\}^n$ into $m +1$ cubical sets are contained in $\RBM_{n,m}$. 
Applying Corollary~\ref{cor:part-mod-max-KL} to such a partition where the cardinality of all blocks is at most $2^{n - \left\lfloor \log(m+1)\right\rfloor}$ yields the bound $D_{\RBM_{n,m}} \leq  n - \left\lfloor \log(m + 1)\right\rfloor$. 
Similarly, using mixtures of product distributions, Theorem~\ref{theorem1} and Corollary~\ref{cor:ind-mix-max-KL} imply the smaller bound $D_{\RBM_{n,m}} \leq n - 1 - \left\lfloor \log(m + 1)\right\rfloor$. 
In this section we derive an improved bound which strictly decreases, as m increases, until 0 is reached. 

\begin{theorem}\label{theorem2}
Let $m<2^{n-1}-1$. Then the maximal Kullback-Leibler divergence from any distribution on $\{0,1\}^n$ to $\operatorname{RBM}_{n,m}$ is upper bounded by 
\begin{equation*}
\max_{p\in\mathcal{P}} D(p\|\operatorname{RBM}_{n,m})
\leq n - \left\lfloor \log(m+1) \right\rfloor - \frac{m+1}{2^{\left\lfloor \log(m+1) \right\rfloor}}
\leq (n -1) - \log (m+1) + 0.1\;.  
\end{equation*}
Conversely, given an error tolerance $0\leq\epsilon\leq 1$, the choice  $m\geq 2^{(n-1)(1-\epsilon) + 0.1}-1$ ensures a sufficiently rich RBM model that satisfies $D_{\RBM_{n,m}}\le \epsilon D_{\RBM_{n,0}}$. 
\end{theorem}

For $m\geq 2^{n-1}-1$ the error vanishes, corresponding to the fact that an RBM with that many hidden units is a universal approximator. 
In Figure~\ref{fig3} we use computer experiments to illustrate Theorem~\ref{theorem2}. 
The proof makes use of the following lemma: 

\begin{figure}
\begin{center}
\includegraphics[clip=true,trim=4cm 21cm 3cm 2.8cm]{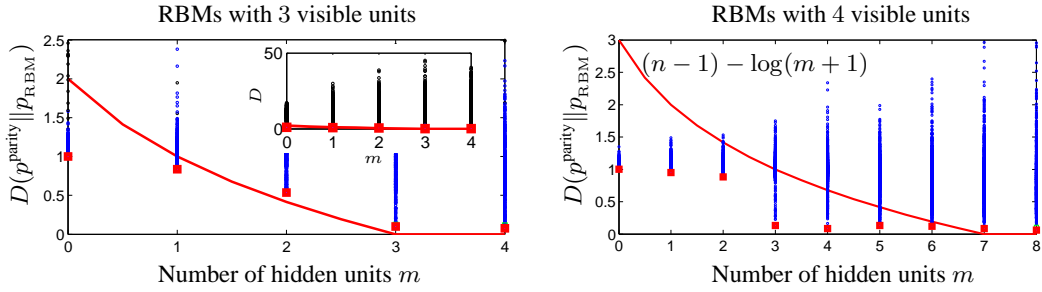}
\end{center}
\vspace{-.2cm}
\caption{%
This figure demonstrates our results for $n=3$ and $n=4$ visible units. The red curves are $(n-1)-\log(m+1)$. 
We fixed $p^{\text{parity}}$ as target distribution, the uniform distribution on binary length $n$ vectors with an even number of ones. 
The distribution $p^{\text{parity}}$ is not the KL-maximizer from $\RBM_{n,m}$, but it is in general difficult to
represent. Qualitatively, samples from $p^{\text{parity}}$ look like uniformly distributed, and representing
$p^{\text{parity}}$ requires the maximal number of product mixture components~\cite{Montufar2010,Montufar2010a}. 
For both values of $n$ and each $m=0,\ldots,2^n/2$ we initialized $500$ resp.~$1000$ RBMs at parameter values chosen uniformly at random in the range $[-10, 10]$. The inset of the left figure shows the resulting KL-divergence $D(p^{\text{parity}}\|p_{\RBM}^{\text{rand}})$ (for $n=4$ the resulting KL-divergence was larger). Randomly chosen distributions in $\RBM_{n,m}$ are likely to be very far from the target distribution. 
We trained these randomly initialized RBMs using CD for $500$ training epochs, learning rate $1$ and a list of even parity vectors as training data. 
The result after training is given by the blue circles. After training the RBMs the result is often not better than the uniform distribution, for which $D(p^{\text{parity}}\|\tfrac{\mathds{1}}{|\{0,1\}^n|})=1$. 
For each $m$, the best set of parameters after training was used to initialize a further CD training with a smaller learning rate (green squares, mostly covered) followed by a short maximum likelihood gradient ascent (red filled squares). 
} \label{fig3}
\end{figure}

\begin{lemma}
  \label{lemmamaxerrormix}
  Let $n_{1},\dots,n_{m}\ge 0$ such that $2^{n_{1}} +\dots+2^{n_{m}}=2^{n}$.  Let $\Mcal$ be the union of all mixtures
  of independent models corresponding to all cubical partitions of $\Xcal$ into blocks of cardinalities $2^{n_{1}},\dots,2^{n_{m}}$. 
  Then $D_{\Mcal} \le \sum_{i: n_{i}>1} \frac{n_{i}-1}{2^{n-n_{i}}}$. 
\end{lemma}

\begin{proof}[Proof of Lemma~\ref{lemmamaxerrormix}]
  The proof is by induction on $n$. 
  If $n=1$, then $m=1$ or $m=2$, and in both cases it is easy to see that the inequality holds (both sides vanish). 
  If $n>1$, then order the $n_{i}$ such that $n_{1}\ge n_{2}\ge\dots\ge n_{m}\ge 0$. Without loss of generality assume $m>1$.

  Let $p\in\Pcal(\Xcal)$, and let $\Ycal$ be a cubical subset of $\Xcal$ of cardinality $2^{n-1}$ such that $p(\Ycal)\le\frac12$. 
  Since the numbers $2^{n_{1}} +\dots+2^{n_{i}}$ for $i=1,\dots,m$ contain all multiples of $2^{n_{1}}$ up to $2^{n}$ and $2^n/2^{n_1}$ is even, there exists $k$ such that $2^{n_{1}} +\dots+2^{n_{k}}=2^{n-1}=2^{n_{k+1}}+\dots+2^{n_{m}}$. 

Let $\Mcal'$ be the union of all mixtures of independence models corresponding to all cubical partitions $\xi=\{\Xcal_1,\ldots, \Xcal_m\}$ of $\Xcal$ into $m$ blocks of cardinalities $n_{1},\dots,n_{m}$ such that $\Xcal_{1}\cup\dots\cup\Xcal_{k}=\Ycal$. 
In the following, the symbol $\sum'_{i}$ shall denote summation over all indices $i$ such that $n_{i}>1$. 
By induction
  \begin{equation}
    D(p\|\Mcal)\le D(p\|\Mcal')
    \le p(\Ycal)\sideset{}{^{\prime}}\sum_{i=1}^{k} \frac{n_{i}-1}{2^{n-1-n_{i}}} 
      + p(\Xcal\setminus\Ycal)\sideset{}{^{\prime}}\sum_{j=k+1}^{m} \frac{n_{j}-1}{2^{n-1-n_{j}}}\;. \label{arrr}
  \end{equation}
There exist $j_{1}=k+1<j_{2}<\dots<j_{k} < j_{k+1} = m + 1$ such that $2^{n_{i}}=2^{n_{j_{i}}}+\dots+2^{n_{j_{i+1}-1}}$ for all $i\leq k$. 
Note that 
\begin{equation*}
    \sideset{}{^{\prime}}\sum_{j=j_{i}}^{j_{i+1}}\frac{n_{j}-1}{2^{n-1-n_{j}}}
    \le \frac{n_{i}-1}{2^{n-1}}(2^{n_{j_{i}}}+\dots+2^{n_{j_{i+1}-1}}) = \frac{n_{i}-1}{2^{n-1-n_{i}}}\;,
\end{equation*}
and therefore
\begin{equation*}
    (\tfrac{1}{2} - p(\Ycal))\frac{n_{i}-1}{2^{n-1-n_{i}}}\,
    +\, (\tfrac{1}{2} - p(\Xcal\setminus\Ycal) ) \sideset{}{^{\prime}}\sum_{j=j_{i}}^{j_{i+1}-1} \frac{n_{j}-1}{2^{n-1-n_{j}}}
    \ge 0\;.
\end{equation*}
Adding these terms for $i=1,\ldots,k$ to the right hand side of equation~(\ref{arrr}) yields
  \begin{align*}
    D(p\|\Mcal)&
    \le \frac12\sideset{}{^{\prime}}\sum_{i=1}^{k} \frac{n_{i}-1}{2^{n-1-n_{i}}} 
      + \frac12\sideset{}{^{\prime}}\sum_{j=k+1}^{m} \frac{n_{j}-1}{2^{n-1-n_{j}}}, 
  \end{align*}
  from which the assertions follow. 
\hfill$\square$\end{proof}

\begin{proof}[Proof of Theorem~\ref{theorem2}]
From Theorem~\ref{theorem1} we know that $\RBM_{n,m}$ contains the union $\Mcal$ of all mixtures of independent models corresponding to all partitions with up to $m+1$ cubical blocks. 
Hence, $D_{\RBM_{n,m}}\le D_{\Mcal}$. 
Let $k = n - \lfloor\log(m+1)\rfloor$ and $l = 2m+2 - 2^{n-k+1}\ge 0$; then $l 2^{k-1} + (m+1-l) 2^{k} = 2^{n}$. 
Lemma~\ref{lemmamaxerrormix} with $n_{1}=\dots=n_{l}=k-1$ and $n_{l+1}=\dots=n_{m+1}=k$ implies
\begin{equation*}
    D_{\Mcal}\le \frac{l(k-2)}{2^{n-k+1}} + \frac{(m+1-l)(k-1)}{2^{n-k}} = k  - \frac{m+1}{2^{n-k}}. 
\end{equation*}
This proves the first inequality. 
For the second inequality see Lemma~\ref{lemma:supplementary} in the Appendix.\footnote{A previous version of this paper erroneously bounded $D_\Mcal$ from above by $(n-1)-\log(m+1)$, which violates the correct bound by a small value (always smaller than $0.1$).}
\hfill$\square$\end{proof}

\section{Conclusion}
\label{sec:conclusion}
We studied the expressive power of the Restricted Boltzmann Machine model with $n$ visible and $m$ hidden units. We presented a hierarchy of explicit classes of probability distributions that an RBM can represent. These classes include large collections of mixtures of $m + 1$ product distributions. In particular any mixture of an arbitrary product distribution and $m$ further product distributions with disjoint supports. The geometry of these submodels is easier to study than that of the RBM models, while these subsets still capture many of the distributions contained in the RBM models. Using these results we derived bounds for the approximation errors of RBMs. We showed that it is always
possible to reduce the error to at most $n - \left\lfloor \log(m+1) \right\rfloor - \frac{m+1}{2^{\left\lfloor \log(m+1) \right\rfloor}}\approx (n-1) - \log(m + 1)$. That is, given any target distribution, there is a distribution within the RBM model for which the Kullback-Leibler divergence between both is not larger than that number. Our results give a theoretical basis for selecting the size of an RBM which accounts for a desired error tolerance. 

Computer experiments showed that the bound captures the order of magnitude of the true approximation error, at least for small examples.  However, learning may not always find the best approximation, resulting in an error that may well exceed our bound. 


\newpage

\bibliography{referenzen}{}\bibliographystyle{abbrv}

\newpage
\appendix
\section{Appendix}

\begin{lemma}\label{lemma:supplementary}
  For all $x>0$,
  \begin{equation*}
    (n-1) -\log(x) \le
    n -\left\lfloor \log(x) \right\rfloor - \frac{x}{2^{\left\lfloor\log (x)\right\rfloor}}
    \leq (n-1) -\log(x) +c\;,
  \end{equation*}
  where $c=-\log(\ln(2)) - (\frac{1}{\ln(2)} -1)\approx 0.086$.
\end{lemma}
\begin{proof}
  Consider the function $f(x) = \log(x) + 1 - \lfloor\log(x)\rfloor - \frac{x}{2^{\lfloor\log(x)\rfloor}}$.  Then $f$ is
  the difference between the concave function $\log(x)$ and a piece-wise linear function interpolating~$\log(x)$.  Hence
  $f$ is non-negative, proving the first inequality.  Moreover, $f(x)=0$ if and only if $x$ is a power of 2.  Between
  each pair of consecutive powers of 2 the function $f$ has a local maximum.  If $x$ is not a power of 2, then $f$ is
  differentiable at $x$ with derivative
  \begin{equation*}
    f'(x) = \frac{1}{2^{\lfloor\log(x)\rfloor}} - \frac{1}{x\ln(2)}.
  \end{equation*}
  This derivative vanishes if and only if $x=2^{\lfloor\log(x)\rfloor}/\ln(2)$.  At such a point,
  \begin{align*}
    f(x)    
    & = \log(2^{\lfloor\log(x)\rfloor}/\ln(2)) + 1 - \lfloor\log(x)\rfloor - \frac{2^{\lfloor\log(x)\rfloor}}{2^{\lfloor\log(x)\rfloor}\ln(2)}
    \\
    & = - \log(\ln(2)) - \frac{1}{\ln(2)} + 1 = c.
  \end{align*}
  Hence, $f(x)\le c$ for all~$x$, proving the second inequality.
\hfill$\square$\end{proof}

Figure~\ref{fig3a} is a supplement to Figure~\ref{fig3}. 

\smallskip

\begin{figure}[h!]
\begin{center}
\setlength{\unitlength}{1cm}
\resizebox{1.03\linewidth}{!}{
\hspace{-.5cm}
\begin{tabular}{ccc}
$D( p^{\text{parity}}\|\RBM_{3,m})$ & $D( p^{\text{parity}}\|\RBM_{4,m})$ & $D( p^{\text{parity}}\|\RBM_{5,m})$\\
\includegraphics[clip=true,trim=4.9cm 9.22cm 4.2cm 9cm,width=5.3cm]{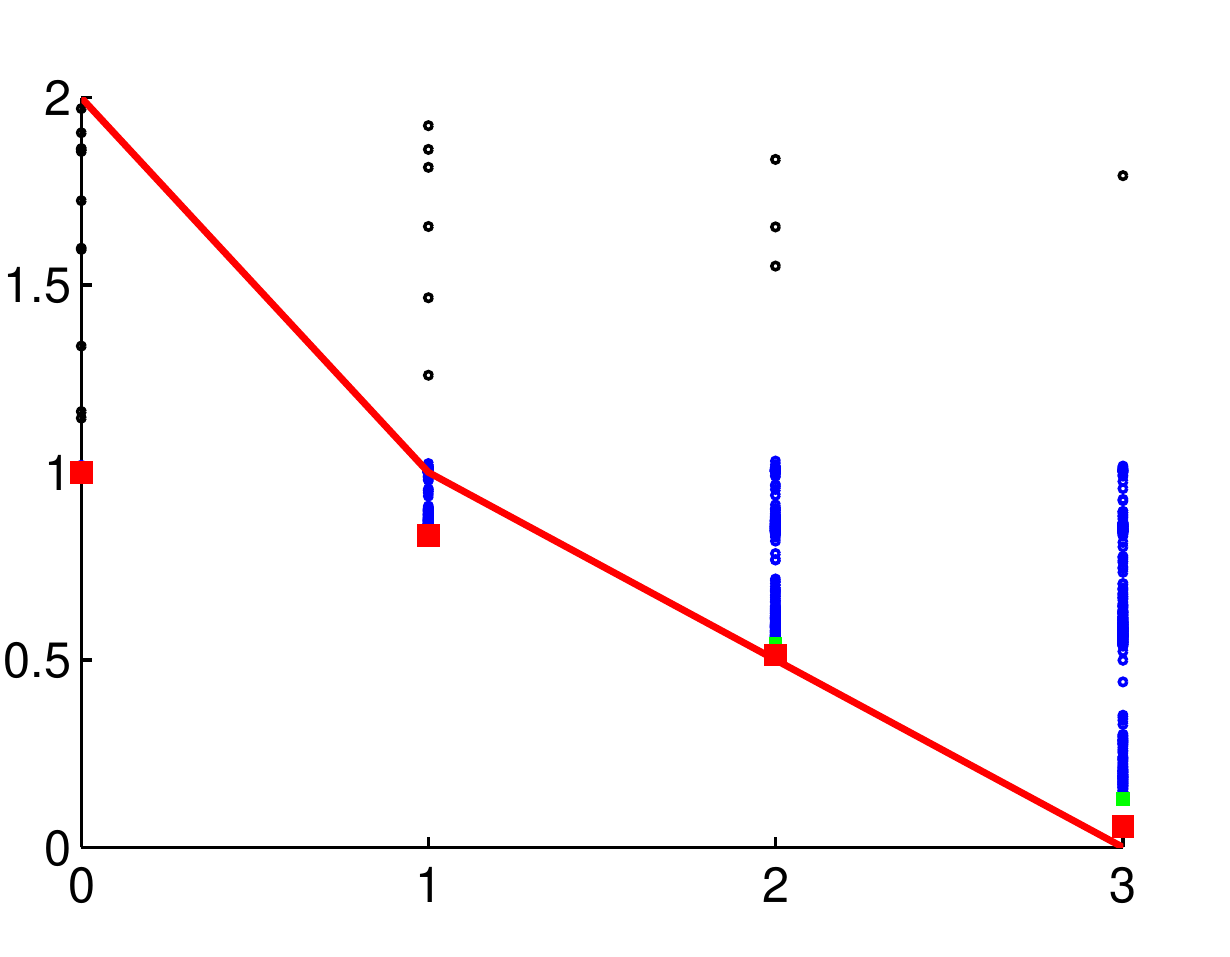}&
\includegraphics[clip=true,trim=4.9cm 9.4cm 4.2cm 9cm,width=5.3cm]{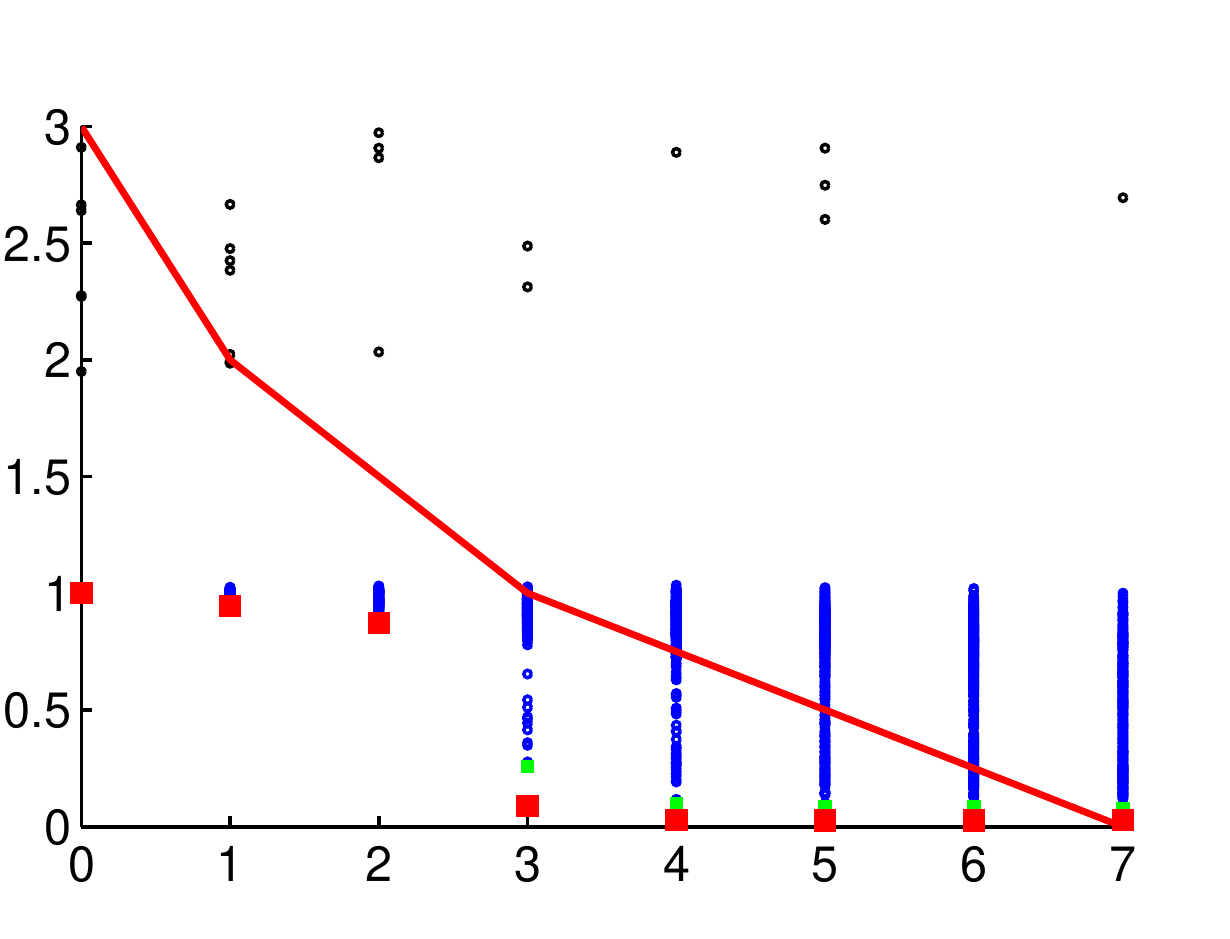}&
\includegraphics[clip=true,trim=4.5cm 8.7cm 4.2cm 9cm,width=5.2cm]{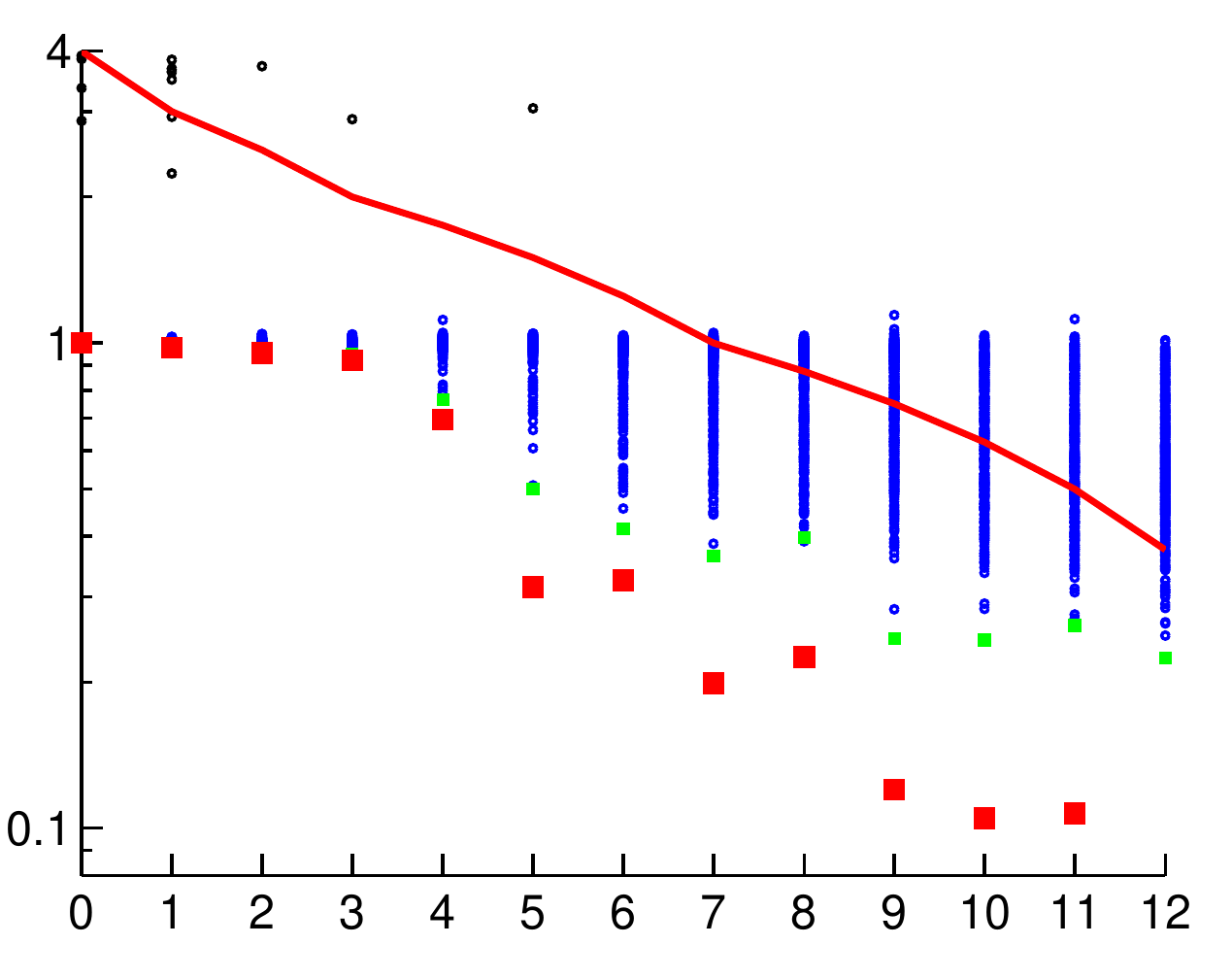}\vspace{-.2cm}\\
$m$&$m$&$m$
\end{tabular}
}
\end{center}
\caption{
This figure demonstrates Theorem~\ref{theorem2} for RBMs with $n=3,4,5$ visible units. 
The red squares show the KL-divergence from the target to the (numerical) best approximation within the RBM model. 
The red curves show the bounds $n - \left\lfloor \log(m+1) \right\rfloor - \frac{m+1}{2^{\left\lfloor \log(m+1) \right\rfloor}}$ from the theorem. 
} \label{fig3a}
\end{figure}

\end{document}